%% file: main.tex
\documentclass{article}
\usepackage[table, svgnames, dvipsnames]{xcolor}
\usepackage{amssymb}
\usepackage{graphicx}

\usepackage[preprint]{corl_2025} %

\usepackage{amsmath,amsfonts}
\usepackage{tcolorbox}
\usepackage{tikz,pgfplots}
\usetikzlibrary{shapes.geometric}
\usetikzlibrary{shapes.symbols}
\usepackage{multirow, makecell} %
\usepackage{wrapfig}
\usepackage{booktabs}

\usepackage{arydshln}
\usepackage{caption}
\usepackage{subcaption}

\usepackage{amsthm}

\newtheorem{lemma}{Lemma}[section]
\newtheorem{remark}{Remark}[section]

\newcommand{\coolname}{CPS}
\newcommand{\netvlad}{NetVLAD} 

\newcommand{\mulran}{MulRan}
\newcommand{\wildplaces}{WildPlaces}
\newcommand{\loggnet}{LoGG3D-Net}
\newcommand{\eg}{\emph{e.g.},}
\newcommand{\ie}{\emph{i.e.},}

\title{A Deeper Look into Second-Order Feature Aggregation for LiDAR Place Recognition}

\author{
  Saimunur Rahman\\
  CSIRO Robotics\\
  Data61, CSIRO,
  Australia\\
  \texttt{saimun.rahman@csiro.au} \\
  \And
  Peyman Moghadam\\
  CSIRO Robotics\\
  Data61, CSIRO, 
  Australia\\
  \texttt{peyman.moghadam@csiro.au} \\
}

\begin{document}
\maketitle

\input{chapters/abstract}


\input{chapters/introduction}

\input{chapters/related_works}

\input{chapters/method}

\input{chapters/experiments}

\input{chapters/conclusion}

\clearpage
\acknowledgments{This work was partially funded by the CSIRO's Machine Learning and Artificial Intelligence (MLAI) FSP and CSIRO's Data61 Science Digital.}

\section*{Limitation}

While \coolname{} remains, to our knowledge, the state-of-the-art second-order aggregation at its dimensionality among methods that fully exploit every backbone channel, below are its limitations.

Because CPS computes covariances only within each channel partition, information about relationships that span different groups is omitted. In practice, this manifests as a modest accuracy penalty: on Oxford Robotcar, raising the partition count from $k=2$ to $k=4$ reduces R@1 by roughly one percentage point. The drop is small because most discriminative structure is preserved inside groups, yet it signals that CPS cannot fully match the representational richness of a complete covariance matrix when very fine-grained cross-channel cues are essential. 

CPS’s power-normalization step depends on the Newton–Schulz iteration for a matrix square root. With the 256-channel backbones used in our experiments, three to five iterations are ample and add only very small time. However, if the backbone grows wider or partitions become very small, the covariance blocks enlarge and additional iterations may be necessary, modestly increasing compute and memory traffic. This iteration overhead is therefore a secondary resource consideration when deploying CPS on extremely constrained processors.

Selecting the partition count $k$ is an extra hyperparameter step. Optimal values shift with backbone width and dataset complexity; $k=2$ is best for our 256-channel models, whereas a slimmer network or a noisier environment may favor $k=4$. Although tuning requires only a small validation sweep, it adds a layer of configuration effort compared with fixed first-order aggregators.

CPS, like any covariance-based method, benefits from having a few points per channel to estimate stable second-order statistics. When scans become exceptionally sparse or heavily corrupted, well below the 4 k-point subsample (\ie{} Oxford and In-house) used in our tests -- the covariance estimate can degrade, and the accuracy margin over first-order aggregation narrows. In practice, this scenario is rare for modern rotating-LiDAR platforms, and simple counter-measures such as merging two partitions or keeping $k$ low restore stability with negligible memory cost.

\bibliography{references}  %

\input{chapters/supplementary}

\end{document}

%% file: chapters/abstract.tex
\begin{abstract}
Efficient LiDAR Place Recognition (LPR) compresses dense point-wise features into compact global descriptors. While first-order aggregators such as GeM and NetVLAD are widely used, they overlook inter-feature correlations that second-order aggregation naturally captures. Full covariance, a common second-order aggregator, is high in dimensionality; as a result, practitioners often insert a learned projection or employ random sketches—both of which either sacrifice information or increase parameter count. However, no prior work has systematically investigated how first- and second-order aggregation perform under constrained feature and compute budgets. In this paper, we first demonstrate that second-order aggregation retains its superiority for LPR even when channels are pruned and backbone parameters are reduced. 
Building on this insight, we propose \underline{C}hannel \underline{P}artition-based \underline{S}econd-order Local Feature Aggregation (\coolname{}): a drop-in, partition-based second-order aggregation module that preserves all channels while producing an order-of-magnitude smaller descriptor. \coolname{} matches or exceeds the performance of full covariance and outperforms random projection variants, delivering new state-of-the-art results with only four additional learnable parameters across four large-scale benchmarks: Oxford RobotCar, In-house, MulRan, and WildPlaces.

\end{abstract}

%% file: chapters/introduction.tex
\section{Introduction}

\begin{wrapfigure}[17]{r}{0.52\linewidth}
\vspace{-12pt}
\centering
\includegraphics[width=\linewidth]{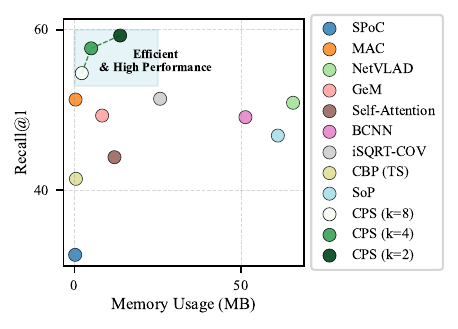}
\caption{R@1 comparison of proposed \coolname{} with common feature aggregation methods with Minkloc3D \citep{komorowski2021minkloc3d} on WildPlaces (Venman Env.) dataset. 
\coolname{} performs higher than others with less memory. }
\label{fig: top_fig}
\vspace{-10pt}
\end{wrapfigure}

LiDAR place recognition (LPR) is the task of identifying previously visited locations from LiDAR scans during the navigation of embodied agents in diverse and dynamic environments, \eg{} self-driving cars, drones, and devices using augmented or virtual reality. LPR methods are frequently formulated as retrieval tasks, leveraging compact descriptors of LiDAR data to enable efficient \citep{komorowski2021minkloc3d, vidanapathirana2022logg3d, komorowski2021egonn} and coarse localization \citep{vidanapathirana2023spectral}. These descriptors serve as global representations of unordered point clouds, typically constructed by aggregating local features using methods such as max pooling or Vector of Locally Aggregated Descriptors (VLAD) \citep{uy2018pointnetvlad, zhang2022kernelized}. For nearly a decade, deep learning based point cloud representation has been the key focus of LPR methods. This has led to developments in various aspects: 3D backbone \cite{uy2018pointnetvlad,komorowski2021minkloc3d,komorowski2022improving,komorowski2021egonn,zhang2019pcan, liu2019lpd, zhou2021ndt}, loss functions \citep{uy2018pointnetvlad,komorowski2022improving, komorowski2021egonn,vidanapathirana2022logg3d}, sequential representation \citep{vidanapathirana2021locus}, as well as global representation \citep{radenovic2018fine,uy2018pointnetvlad,komorowski2021minkloc3d}. Briefly, these learning based global representation methods (also commonly known as pooling methods) aggregate local features through maximization (\eg{} MAC \citep{radenovic2018fine}), averaging (\eg{} SPoC \citep{radenovic2018fine}), exponentiation (\eg{} GeM \citep{radenovic2018fine}), and VLAD (\eg{} NetVLAD \cite{arandjelovic2016netvlad}). State-of-the-art LPR methods commonly use GeM or VLAD, depending on their design. While the VLAD is more memory-intensive, both aggregation methods produce a lower dimensional global descriptor and encode first-order statistics of local descriptors.

GeM and VLAD, representative examples of First-order Aggregation (FoA) methods, aggregate local features by capturing their average or weighted contributions, showing strong performance in LPR. However, FoA methods treat local features independently and ignore their inter-relationships. To address this, Second-order Aggregation (SoA) methods were introduced to capture richer representations by modeling pairwise feature dependencies \citep{zhang2022kernelized, vidanapathirana2021locus}.  
However, SoA’s main drawback is its quadratic output dimensionality: for $d$ local features, it yields a $d \times d$ dimensional covariance matrix or $\frac{d(d+1)}{2}$ leveraging symmetry. This presents practical challenges in LPR, where descriptors need to be stored and compared across large-scale maps. High-dimensional descriptors increase memory usage, slow down retrieval due to expensive distance computations, and place a heavy load on storage and communication resources. Additionally, large descriptor sizes can lead to overfitting \citep{li2017factorized}, especially in scenarios with limited training data or high environmental variability. %

To address these limitations, we propose a new method that produces lower dimensional second-order global features.%
We call this method \underline{C}hannel \underline{P}artition-based \underline{S}econd-order Local Feature Aggregation (CPS). 
Theoretically, it reduces the overall output dimensionality by approximately a factor of $k$, \ie{} $\frac{\frac{d}{k}(\frac{d}{k}+1)}{2}$, compared to computing a full $d \times d$ covariance matrix, while still capturing rich intra-channel dependencies. CPS thus achieves lower output dimensionality, retains 
local feature information, and efficiently models second-order relationships. 
Fig. \ref{fig: top_fig} shows a comparison of three \coolname{} variants with common aggregation methods. We conduct extensive experiments with our proposed SoA method on four popular LPR datasets, namely, Oxford Robotcar \cite{maddern20171}, In-house \cite{uy2018pointnetvlad}, \mulran~\cite{kim2020mulran} and \wildplaces~\cite{knights2023wild}, to demonstrate its effectiveness. We also compare with the existing state-of-the-art, common FoA and SoA methods used in LPR and three popular covariance pooling methods from the computer vision domain \cite{gao2016compact,li2018towards}. %

%% file: chapters/related_works.tex
\section{Related works}
\paragraph{First-order Methods:} Most LPR approaches rely on first-order pooling \cite{uy2018pointnetvlad,xia2021soe,liu2019lpd,xu2021transloc3d,zhang2019pcan,wiesmann2022retriever,wiesmann2022kppr, lin2023se}. PointNetVLAD \cite{uy2018pointnetvlad} introduced 3D feature pooling by integrating PointNet \cite{qi2017pointnet} with NetVLAD \cite{arandjelovic2016netvlad}, using linear projection to reduce NetVLAD’s high-dimensional output. However, the large number of parameters can cause overfitting in low-data scenarios. GeM pooling \cite{komorowski2021minkloc3d,knights2023geoadapt} offers compactness but is sensitive to outliers and limited by a single learnable parameter. Other works have employed global max \cite{barros2022attdlnet} or average pooling \cite{lai2022adafusion}, often augmented with cross-attention. In contrast, \coolname{} avoids linear projection, but can be used to further reduce its dimensionality. Each group-wise covariance matrix is regularized independently, offering more flexibility than GeM. There are PatchNetVLAD \cite{hausler2021patch}, MixVPR \cite{ali2023mixvpr}, and SALAD \cite{izquierdo2024optimal} methods proposed in Visual Place Recognition (VPR) in 2D, their applicability to LPR remains underexplored to date. Our work focuses on LPR, which presents unique challenges such as sparse point clouds and variable input sizes.  

\paragraph{Second-order Methods:} Widely used in 2D fine-grained recognition \cite{koniusz2018deeper,koniusz2021power}, second-order methods like Bilinear CNN \cite{lin2015bilinear}, its variants \cite{lin2017improved,li2017second}, and compact alternatives like CBP \cite{gao2016compact} leverage covariance pooling. A common approach used in most of these works to mitigate the dimensionality challenge is to apply linear projection, \ie{} $1\times 1$ convolution, before aggregation \citep{gao2016compact, li2018towards}. Here, the number of local features is reduced using a learnable transformation, which compresses them before computing the global representation. This strategy is driven by the empirical observation that SoA methods often maintain strong performance even when operating with fewer input descriptors. Following this insight, \loggnet{} \citep{vidanapathirana2022logg3d} adopts the same practice to produce compact global descriptors while keeping computational overhead low. In this work, we revisit and evaluate this popular strategy within both FoA and SoA settings for LPR. We ask: \textit{How does descriptor reduction via linear projection affect performance in each case?} Our results (as shown in Fig. \ref{fig: performance-vs-descriptor-count}) reveal a compelling and previously unreported finding -- SoA methods consistently outperform FoA even under significant descriptor size reduction. This robustness to reduced descriptor sizes, which may be known in the image domain, has not been studied in LPR before. Compact bilinear pooling (CBP) \cite{gao2016compact}, along with its subsequent variants, provides an alternative direction by approximating high-dimensional SoA outputs without explicitly reducing local features. However, these methods were originally developed for the image domain and have not been explored in the context of LPR. Moreover, our evaluation of CBP’s original, highest-performing configuration yields suboptimal results in this domain (please see Tab. \ref{tab: soa-pooling-comparisons} for detailed results). In contrast, our method requires no projections, making it more adaptable and integration-friendly for deep neural networks.

In LPR, Locus \cite{vidanapathirana2021locus} introduced second-order pooling via per-descriptor covariance matrices, which was then extended by \loggnet{} \cite{vidanapathirana2022logg3d} to end-to-end descriptor learning. Despite its strong performance, \loggnet{} is memory-intensive due to per-point covariance computations that scale with the size of the input point cloud. To mitigate this, \loggnet{} uses fewer feature channels; however, most modern networks employ significantly more, which raises scalability concerns. Additionally, it relies on eigen-decomposition-based matrix normalization, which is not GPU-friendly and is prone to instability issues~\cite{li2017second, engin2018deepkspd}. 
Our approach addresses this by computing a single channel-wise covariance matrix, significantly reducing memory demands. It employs GPU-friendly normalization and avoids numerical instability. 

While the above methods compute SoA via the covariance matrix which captures linear correlations, \citet{zhang2022kernelized} introduced SoA via the kernel matrix which captures non-linear correlations. However, they use it to enrich max-pooled representations. Despite the hybrid nature of their method, we isolate their kernel component and assess its ability to capture non-linear correlations using our \coolname{} matrix normalization. Results are provided in the Supplementary Material. We mention the covariance can be considered as a special case of the kernel matrix when a linear kernel is used; therefore, theoretically, our proposed \coolname{} can capture even better second-order information with kernel matrices. We leave the exploration in this direction for our future work.

%% file: chapters/method.tex
\section{Proposed Method: Basics of SoA for LiDAR and Detail of \coolname{}}
The aim of \coolname{} is to learn a low-dimensional global representation $\mathbf{z}$ of a point cloud $\mathbf{P}^{N\times 3}$, consisting of $N$ unordered points, for LPR (as shown in Fig.~\ref{fig: main figure}). We formulate second-order aggregation (SoA) for LPR following the Improved BCNN framework \citep{lin2017improved}, and subsequently introduce CPS to address its limitations. Our approach is inspired by SoA techniques in the image domain \citep{lin2015bilinear, gao2016compact, lin2017improved, li2017second, lin2018second, koniusz2018deeper, koniusz2021tensor, koniusz2021power}, as well as recent advances in SoA for LPR \citep{zhang2022kernelized}. We begin with a basic formulation of SoA for LiDAR point clouds before presenting our method.

\paragraph{Basics of second-order local point feature aggregation:}
Let $\mathbf{X}_{d\times N} = [x_1, x_2, ..., x_{N}]$ be a data matrix containing $N$ columns of $d$-dimentional local descriptors extracted by a deep network $\phi(\cdot)$ from point cloud $\mathbf{P}$ (\ie{} we use a sparse convolution based network indentical to \citep{komorowski2021minkloc3d,komorowski2022improving} in our experiments). Following \cite{lin2015bilinear}, a sample covariance matrix $\mathbf{C}$ of $\mathbf{X}$ is computed as
\begin{equation}
    \mathbf{C}_{d\times d} = \frac{1}{N} \bar{\mathbf{X}} \bar{\mathbf{X}}^\top, \quad \text{where } \bar{\mathbf{X}} = \mathbf{X} - \boldsymbol{\mu}, \quad \boldsymbol{\mu} = \frac{1}{M} \sum_{i=1}^{M} \mathbf{x}_i,
    \label{eq: covariance matrix}
\end{equation}
where $\bar{\mathbf{X}}$ denotes centered $\mathbf{X}$ and $\top$ denotes transpose. The $(i,j)$-th entry of $\mathbf{C}$ represents the correlation among $d$ components. Since $\mathbf{C}$ is a symmetric positive definite matrix, matrix normalization such as power normalisation (PN) \cite{lin2017improved,li2017second} is often applied to respect its Riemannian geometry and to further boost its representational ability to combat feature burstiness -- a phenomenon where the same feature appears many times. PN is traditionally done via matrix decomposition \citep{ionescu2015matrix}. Let us obtain eigenvectors $\mathbf{U}$ and eigenvalues $\mathbf{D}$ of $\mathbf{C}$ via eigendecompostion as $\mathbf{C}=\mathbf{U}\mathbf{D}\mathbf{U}^\top$. The PN can be performed to $\mathbf{C}$ by converting the matrix power to the power of eigenvalues. Briefly, the process is defined as
\begin{equation}
    \mathbf{C}^\alpha = \mathbf{U}\mathbf{D}^\alpha\mathbf{U}^\top, \quad \text{where } \alpha=0.5.
\end{equation}
Several works have improved the numerical stability of end-to-end PN, \eg{} \citep{li2017second,engin2018deepkspd}. After applying PN, the resultant matrix (or its upper triangular entries due to the matrix being symmetric) can be used as a global representation of $\mathbf{P}$ for the LPR task. Eq. (\ref{eq: covariance matrix}) clearly shows that the size of $\mathbf{C}$ increases as the size $d$ (\ie{} number of feature channels) increases. Furthermore, it remains permutation invariant (please see the Supplementary Materials for more details).

Instead of reducing the size of $d$ with linear projection such as $1\times 1$ convolution to obtain a smaller $\mathbf{C}$ as in prior works, CPS computes SoA without reducing $d$ following below sequential steps:

\begin{figure}
    \centering
    \includegraphics[width=1\linewidth]{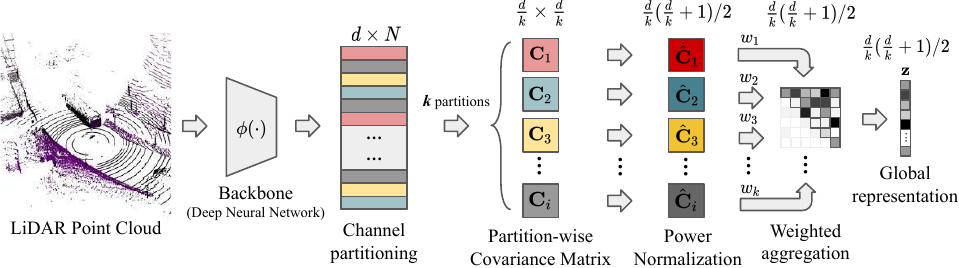}
    \caption{Overview of CPS. Given descriptors $\mathbf{X} \in \mathbb{R}^{d\times N}$ from a 3D backbone $\phi(\cdot)$, we partition them into $k$ disjoint groups, compute a normalized second-order, \ie{} covariance, matrix per group, and aggregate the upper-triangular entries to form the lower dimensional global representation $\mathbf{z}$.} 
    \label{fig: main figure}
\end{figure}

\paragraph{Channel partitioning:}  
Instead of computing the covariance matrix over the entire feature matrix, CPS divides the feature matrix $\mathbf{X} \in \mathbb{R}^{d \times N}$ along its channel (row) dimension into $k$ non-overlapping submatrices $\mathbf{X}_1, \mathbf{X}_2, \ldots, \mathbf{X}_k$, where each $\mathbf{X}_i \in \mathbb{R}^{d_i \times N}$ and $\sum_{i=1}^{k} d_i = d$. This yields the decomposition $\mathbf{X} = [\mathbf{X}_1, \mathbf{X}_2, \ldots, \mathbf{X}_k]^\top$, with each submatrix representing a distinct group of channels. We denote this operation as $\mathtt{partition}(\mathbf{X}, k) = \{ \mathbf{X}_1, \ldots, \mathbf{X}_k \}$.

\paragraph{Group-wise covariance matrix computation:}  
For each submatrix $\mathbf{X}_i \in \mathbb{R}^{d_i \times N}$ obtained through channel partitioning, we compute an empirical covariance matrix $\mathbf{C}_i \in \mathbb{R}^{d_i \times d_i}$ defined as:
\begin{equation}
    \mathbf{C}_i = \frac{1}{N} \, \bar{\mathbf{X}}_i \, \bar{\mathbf{X}}_i^\top, \quad \text{where } \bar{\mathbf{X}}_i = \mathbf{X}_i - \boldsymbol{\mu}_i \mathbf{1}^\top,
\end{equation}
with $\boldsymbol{\mu}_i = \frac{1}{N} \sum_{j=1}^{N} \mathbf{x}^{(i)}_j \in \mathbb{R}^{d_i}$ denoting the mean vector of the $i$-th group and $\mathbf{1} \in \mathbb{R}^N$ denoting a column vector of ones. Each column $\mathbf{x}^{(i)}_j \in \mathbb{R}^{d_i}$ corresponds to the $j$-th sample in group $i$. This formulation ensures that $\bar{\mathbf{X}}_i$ is column-centered prior to covariance computation

\paragraph{Efficient normalization of group-wise covariance matrix:}
To obtain robust feature representations, we apply matrix power normalization to each group-wise covariance matrix $\mathbf{C}_i \in \mathbb{R}^{d_i \times d_i}$. While conventional approaches rely on eigenvalue decomposition to compute the matrix square root $\mathbf{C}_i^{1/2}$, such methods are incompatible with efficient GPU parallelism \citep{li2017second,li2018towards,koniusz2018deeper}. As an alternative, we adopt the Newton-Schulz (NS) iteration \citep{li2017second,li2018towards}, which provides a numerically stable and GPU-friendly approximation of matrix square roots using only matrix multiplications. Prior to iteration, we perform pre-normalization by the trace to ensure bounded spectral norm: $\tilde{\mathbf{C}}_i = \mathbf{C}_i/\operatorname{tr}(\mathbf{C}_i)$. The iteration is then initialized as: $\mathbf{Y}_0 = \tilde{\mathbf{C}}_i, \mathbf{Z}_0 = \mathbf{I}$, and proceeds for $T$ steps using:
\begin{equation}
    \mathbf{Y}_{t+1} = \frac{1}{2} \mathbf{Y}_t (3\mathbf{I} - \mathbf{Z}_t \mathbf{Y}_t), \quad
\mathbf{Z}_{t+1} = \frac{1}{2} (3\mathbf{I} - \mathbf{Z}_t \mathbf{Y}_t) \mathbf{Z}_t,
\quad \text{for } t = 0, 1, \ldots, T-1.
\end{equation}
The output \( \mathbf{Y}_T \) serves as an approximation of \( \tilde{\mathbf{C}}_i^{1/2} \). To preserve consistency across groups, we apply post-normalization by trace: $\hat{\mathbf{C}}_i = \sqrt{\operatorname{tr}(\mathbf{C}_i)} \cdot \mathbf{Y}_T.$
The upper triangular entries of the resulting normalized covariance matrices \( \hat{\mathbf{C}}_i \) are then used as final representations for each channel group. 

\paragraph{Weighted aggregation of normalized second-order features:}
Let \( \mathbf{c}_i \in \mathbb{R}^{d_i(d_i+1)/2} \) denote the vector containing the upper triangular entries (including the diagonal) of the normalized group-wise covariance matrix \( \hat{\mathbf{C}}_i \). These vectors capture the second-order statistics of each channel group. To obtain a compact global representation, we compute a weighted sum over all \( k \) groups: $\mathbf{z} = \sum_{i=1}^{k} w_i \mathbf{c}_i$, where \( w_i \in \mathbb{R} \) are learnable scalar weights optimized end-to-end. This aggregation scheme allows the model to adaptively emphasize more informative group-level second-order descriptors, enhancing the expressiveness of the final representation \( \mathbf{z} \).

It is important to note that the covariance computation of partitioned channel groups in \coolname{} remains permutation invariant. The NS iteration based power normalization process and weighted aggregation do not alter that. Therefore, \coolname{} produces permutationally invariant representations, which is important for LPR. The channel partitioning preserves correlations within the group but sacrifices inter-group correlations as a trade-off between compactness and performance.

\begin{remark}[Second-order Statistics Coverage of \coolname{}]\itshape
\label{rm: sop-converge}
Let \(X\in\mathbb R^{C\times N}\) be a layer output with centred covariance \(\Sigma_X\).
Partition the \(C\) channels into \(k\) equal groups, yielding blocks
\(\{X^{(i)}\}_{i=1}^{k}\).
The CPS descriptor is
\(\mathbf z=\sum_{i=1}^{k} w_i\,\operatorname{vec}_{\text{upper}}\!\bigl(\Sigma_{X^{(i)}}^{1/2}\bigr)\)
with \(w_i\ge 0\) and \(\sum_{i} w_i = 1\).
Hence \(\mathbf z\) preserves exactly the block-diagonal of \(\Sigma_X\)—that is,
all covariances between channels within the same group—while reducing the
dimensionality from \(O\!\bigl(C(C+1)/2\bigr)\) to
\(O\!\bigl((C/k)((C/k)+1)/2\bigr)\).
When \(k = 1\) the full covariance is retained, so no second-order information
is lost; for \(k > 1\) only the off-block entries are discarded, giving a
tunable trade-off between performance and memory.
\end{remark}

%% file: chapters/experiments.tex
\section{Experiments}

We conduct a five-part empirical study on covariance matrix-based SoA for LPR, systematically evaluating its performance and efficiency. The first four investigations are carried out using \textbf{MinkLoc3D} backbone \citep{komorowski2021minkloc3d}: (1) we quantify the aggregation gains of the full covariance matrix over FoA methods; (2) we stress-test it under constrained resource budgets via projection and backbone slimming which show that SoA remains robust even when capacity is severely limited; (3) we evaluate our proposed method \coolname{} across different values of $k \in {2, 4, 8, 16}$, finding that $k{=}2$ yields optimal performance; %
and (4) we compare \coolname{} ($k{=}2$) against four leading SoA variants, \ie{} Bilinear CNN (BCNN) \citep{lin2015bilinear}, iSQRT-COV (iSQRT) \citep{li2017second} and CBP~\citep{gao2016compact}, and \loggnet{} (SoP) ~\citep{vidanapathirana2022logg3d} in terms of size, memory, and accuracy, showing that CPS as the most size-efficient among SoA methods and outperforms CBP by a significant margin. Finally, (5) we integrate \coolname{} ($k{=}2$) into \textbf{MinkLoc3Dv2}, achieving new state-of-the-art results on the Oxford, In-house, MulRan, and WildPlaces datasets. All experiments share a consistent training pipeline and evaluation protocol to ensure fair comparisons. \coolname{} ($k{=}2$) demonstrates the best accuracy-efficiency trade-off to date, operating at just one-quarter the size of the full covariance approach. We use PyTorch and MinkowskiEngine \citep{choy20194d}, and run all experiments on a computing cluster with Nvidia H100 GPUs. 

\paragraph{Datasets and their evaluation criteria:} 
We conduct extensive experiments with four large-scale, public datasets for LiDAR place recognition, namely, \textbf{Oxford Robotcar} \citep{maddern20171},  \textbf{In-house}  \citep{uy2018pointnetvlad}, \textbf{\mulran}~\citep{kim2020mulran} and \textbf{\wildplaces}~\citep{knights2023wild} to demonstrate the effectiveness of \coolname{}. The details of the datasets are given in the Supplementary Material. For Oxford Robotcar and In-house datasets, we use the protocols, training, and testing splits introduced by PointNetVLAD~\cite{uy2018pointnetvlad}. For \mulran, we follow the training and testing splits used in \cite{komorowski2021egonn,vidanapathirana2023spectral}. For \wildplaces{}, we follow the inter-sequence training and testing protocol introduced in \citep{knights2023wild}. We report Recall@1 (R@1) for all datasets, plus Recall@1\% (R@1\%) on Oxford/In-house, Recall@5 (R@5) on MulRan, and Mean Reciprocal Rank (MRR) on WildPlaces. All experiments share this unified protocol to guarantee consistent comparison.

\paragraph{Study 1: Benchmarking full covariance against first-order aggregation:}
We first study the contribution of pure SoA by comparing full-covariance matrix aggregation (for this, we choose iSQRT-COV \citep{li2018towards} due to its wide use in literature. It is worth noting that iSQRT-COV is a special case of \coolname{} with $k=1$. More details on $k$ are in Study 3.) to four popular FoA methods used in LPR, \ie{} SPoC \citep{radenovic2018fine}, MAC \citep{radenovic2018fine}, NetVLAD \citep{arandjelovic2016netvlad}, and GeM \citep{radenovic2018fine}, using the unchanged 256-dimensional (256-d), \ie{} number of output feature channels, MinkLoc3D backbone.  All models are trained on Oxford using the baseline protocol \citep{uy2018pointnetvlad} and evaluated on three In-house regions \citep{uy2018pointnetvlad}.  As shown in the 256-d column of Table~\ref{tab: recall1_baseline_grouped}, on Oxford RobotCar SPoC, MAC, NetVLAD and GeM achieve R@1 of 81.1\%, 92.0\%, 89.6\% and 93.5\%, respectively, while iSQRT-COV raises this to 94.0\%, a 0.5\,pp gain over GeM and a 2.0\,pp gain over MAC.  Importantly, however, this advantage grows markedly in the more challenging In-house benchmarks: University Sector (U.S.)\ improves from 87.2\,\% to 88.4\,\% (+1.2\,pp), Residential Area (R.A.)\ from 79.9\,\% to 85.3\,\% (+5.4\,pp), and Business District (B.D.)\ from 81.6\,\% to 85.0\,\% (+3.4\,pp).  These larger, consistent gains in diverse environments demonstrate that full covariance captures subtle cross-channel interactions that first-order means miss, establishing a clear performance baseline for our subsequent compactness and efficiency studies.

\begin{table*}[t]
\setlength{\tabcolsep}{3.5pt}
\centering
\caption{Comparison of R@1 performance between FoA and SoA methods under reduced backbone output feature dimensionality ($d$), excluding $1{\times}1$ convolutions. Results are grouped by dataset and feature dimension. Reducing channel count decreases the backbone’s parameter size: \textit{16~$\rightarrow$~0.50M, 32~$\rightarrow$~0.51M, 64~$\rightarrow$~0.54M, 128~$\rightarrow$~0.65M, 256~$\rightarrow$~1.06M.}
}
\resizebox{\linewidth}{!}{
\begin{tabular}{lccccc ccccc ccccc ccccc}
\toprule
\multirow{2}{*}{\makecell[c]{\textbf{Aggregation}\\\textbf{Technique}}} & \multicolumn{5}{c}{\textbf{Oxford}} & \multicolumn{5}{c}{\textbf{University Sector (U.S.)}} & \multicolumn{5}{c}{\textbf{Residential Area (R.A.)}} & \multicolumn{5}{c}{\textbf{Business District (B.D.)}} \\ \cmidrule(r){2-6} \cmidrule(r){7-11} \cmidrule(r){12-16} \cmidrule(r){17-21}
 & 16-d & 32-d & 64-d & 128-d & 256-d & 16-d & 32-d & 64-d & 128-d & 256-d & 16-d & 32-d & 64-d & 128-d & 256-d & 16-d & 32-d & 64-d & 128-d & 256-d \\
\midrule
SPoC     & 77.6 & 81.0 & 81.2 & 80.9 & 81.1 & 66.2 & 71.5 & 72.2 & 71.0 & 69.2 & 60.0 & 62.2 & 66.7 & 62.3 & 59.6 & 61.7 & 66.9 & 65.7 & 66.3 & 64.9 \\
MAC     & 1.6  & 73.8 & 83.3 & 89.9 & 92.0 & 2.0  & 55.5 & 70.7 & 79.8 & 85.3 & 2.5  & 41.2 & 61.6 & 73.6 & 82.7 & 0.7  & 53.0 & 67.4 & 77.6 & 83.3 \\
NetVLAD & 74.9 & 85.7 & 89.8 & 90.5 & 89.6 & 61.5 & 80.0 & 81.3 & 82.7 & 81.8 & 52.9 & 70.6 & 77.6 & 79.2 & 75.1 & 57.9 & 71.3 & 75.9 & 81.3 & 75.8 \\
GeM     & 78.9 & 86.7 & 91.8 & 93.0 & 93.5 & 63.7 & 74.0 & 80.1 & 85.2 & 87.2 & 56.6 & 65.2 & 74.9 & 77.6 & 79.9 & 58.5 & 68.2 & 73.6 & 79.7 & 81.6 \\
\rowcolor{green!15}
iSQRT-COV  & \textbf{92.7} & \textbf{93.9} & \textbf{94.2} & \textbf{93.6} & \textbf{94.0} & \textbf{88.7} & \textbf{91.0} & \textbf{88.1} & \textbf{89.8} & \textbf{88.4} & \textbf{84.1} & \textbf{83.9} & \textbf{87.4} & \textbf{86.1} & \textbf{85.3} & \textbf{83.8} & \textbf{84.6} & \textbf{86.4} & \textbf{85.0} & \textbf{85.0} \\

\bottomrule
\end{tabular}}
\label{tab: recall1_baseline_grouped}
\end{table*}

\paragraph{Study 2: Effectiveness of SoA under constrained feature budgets:}
We tested whether SoA retains its effectiveness under tight descriptor budgets via two experiments: one projecting features on the backbone as in literature \citep{li2017second,li2018towards}, the other slimming it down, with all training settings unchanged.

\begin{wrapfigure}[14]{r}{0.40\linewidth}
\vspace{-10pt}
\centering
\includegraphics[width=\linewidth]{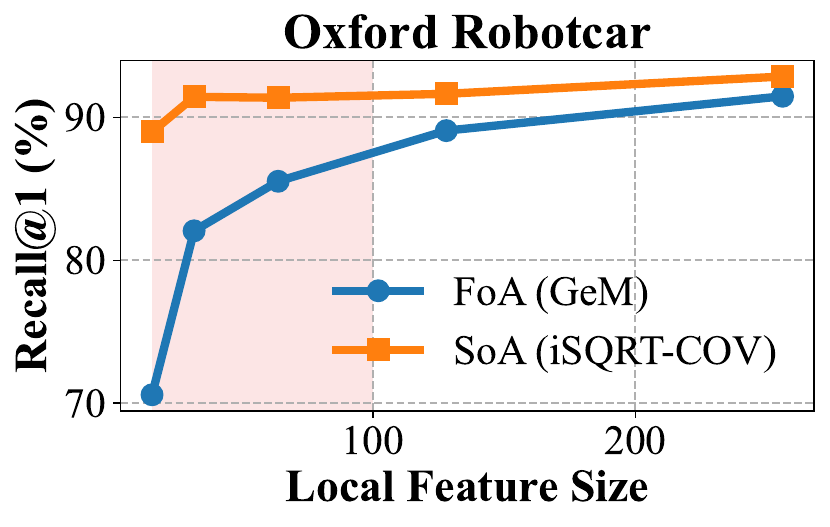}
\caption{Performance of GeM and iSQRT-COV when local feature channels $d$ are reduced with $1\times 1$ convolution to following dimensions: \ie{} $d=\{16, 32, 64, 128, 256\}$.}
\label{fig: performance-vs-descriptor-count}
\vspace{-10pt}
\end{wrapfigure}

In the \textit{projection-based experiment}, we insert a learnable $1\times1$ convolution after MinkLoc3D’s final ResNet block to project the 256-d feature map down to $d=\{16,32,64,128\}$.  As shown by the red-highlighted region in Fig.~\ref{fig: performance-vs-descriptor-count}, when $d=16$ the full covariance method (iSQRT-COV) sustains 89.0\,\% R@1 on Oxford -- only a few points below its full-width performance, whereas GeM collapses to 70.6\,\%, creating an 18.4\,pp gap.  Even at $d=32$, SoA outperforms FoA by over 7\,pp, confirming that covariance statistics pack substantially more discriminative information into each channel than mean-based aggregation. In the \textit{backbone-slimming experiment}, we remove the projection layer entirely and instead reconstruct every convolution in MinkLoc3D so that its \texttt{feature\_size} is set directly to $d$, naturally reducing both the feature dimensionality and the total network parameters.  As reported in Tab.~\ref{tab: recall1_baseline_grouped}, at $d=16$ iSQRT-COV achieves 92.7\,\% R@1 versus 78.9\,\% for GeM, and it maintains a consistent multi-point lead at $d=32$ and $d=64$.  This structural test confirms that the SoA advantage arises from the richer covariance representation rather than from extra projection weights. %

Both stress tests demonstrate that SoA is intrinsically more “bandwidth-efficient” than FoA: it preserves critical cross-channel cues even when the channel count is cut by an order of magnitude or when the entire network is slimmed.  This robustness makes covariance matrix aggregation particularly attractive for resource-constrained robotics platforms, where memory and compute budgets are tight but reliable place recognition remains essential. These findings motivate our next study, where we aim to retain this resilience but with a reduced size covariance descriptor itself, without bells and whistles via the idea of channel partitioning.

\begin{wrapfigure}[19]{r}{0.48\linewidth}
\vspace{-10pt}
\centering
\includegraphics[width=\linewidth]{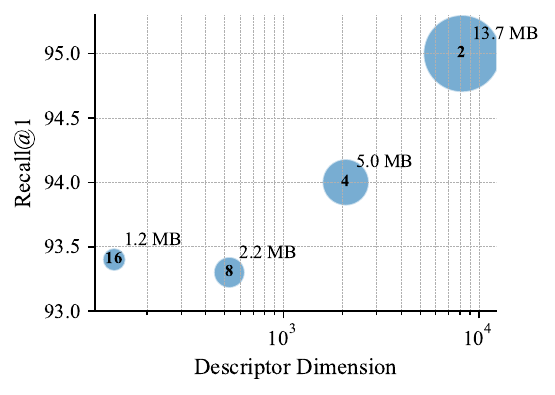}
\caption{Compact visualization of CPS trading off descriptor size and memory (bubble area) across partition sizes $k$ (marked inside bubbles) to maintain high Oxford R@1, even at $k=16$ (136-d, 1.2 MB), achieving 93.4\% recall with over 60x dimensionality reduction and $>$10x lower memory vs. $k=2$.}
\label{fig: performance-vs-dim}
\vspace{-10pt}
\end{wrapfigure}

\paragraph{Study 3: \coolname{} delivers compact global descriptor with minimal loss of accuracy:}
We next examine how the covariance descriptor effectively can be compressed without sacrificing accuracy by sweeping our CPS method over $k=\{2,4,8,16\}$ while keeping the MinkLoc3D backbone fixed with 256-d channels (see Fig.~\ref{fig: performance-vs-dim}). As $k$ shrinks the descriptor size from 8k dimension (\ie{} $k=2$) to only 136 dimension; GPU memory falls from 13.7 MB to 1.2 MB accordingly. Throughout this range Oxford Robotcar R@1 hardly budgets: it sits at 95\% for $k=2$, stays 94\% for $k=4$, even under the most aggressive compression ($k=16$, a 60x reduction relative to the $k=2$), reduces just 0.5 pp performance. The ability to trade two orders of magnitude in memory for only a single-digit percentage point of accuracy positions CPS as a drop-in solution for both high-performance servers and ultra-lightweight robotic platforms.

\begin{wraptable}[12]{r}{0.5\linewidth}
\vspace{-10pt}
\setlength{\tabcolsep}{2pt}
\centering
\caption{Comparison of SoA methods on the Oxford and In-house datasets with baseline protocol. All methods use 256-d local feature channels.}%
\label{tab: soa-pooling-comparisons}
\resizebox{\linewidth}{!}{%
\begin{tabular}{lcccccccc}
\toprule
& \textbf{Dim.} & \makecell[c]{\textbf{Memory}\\\textbf{(MB)}} & \textbf{Oxford} & \textbf{U.S.} & \textbf{R.A.} & \textbf{B.D.} \\
\midrule
BCNN \cite{lin2015bilinear} & 64k & 51.3 & 92.8 & 81.0 & 75.1 & 76.7 \\
iSQRT \cite{li2018towards} & 32k & 25.7 & 94.0 & 88.4 & 85.3 & 85.0 \\
CBP (TS) \cite{gao2016compact} & 8k & 0.44 & 84.2 & 71.2 & 64.2 & 67.0 \\
SoP \cite{vidanapathirana2022logg3d} & 64k & 61.04 & 87.1 & 73.7 & 64.9 & 69.3 \\
\midrule
\rowcolor{green!15}
\textbf{\coolname{}} ($k=2$) & 8k & 13.7 & \textbf{95.0} & \textbf{90.1} & \textbf{86.2} & \textbf{85.0} \\
\bottomrule
\end{tabular}
}
\vspace{-10pt}
\end{wraptable}

\paragraph{Study 4: Comparison of CPS with established SoA methods:}
We now position our compact configuration, \textit{CPS ($k=2$)}, against four widely used second-order alternatives, from image domain, BCNN \citep{lin2015bilinear}, iSQRT-COV \citep{li2017second} and CBP \citep{gao2016compact}, and from LPR, \loggnet{}'s SoP \citep{vidanapathirana2022logg3d} using the same MinkLoc3D backbone and training pipeline as above. Note that the SoA methods from the image domain are being evaluated in LPR for the first time. We follow their original implementation to reimplement them within the MinkLoc3D codebase. The results, shown in the Tab.~\ref{tab: soa-pooling-comparisons}, contrast each method’s descriptor dimensionality, peak GPU scratch memory, and R@1 on Oxford and the three In-house regions.  BCNN and SoP deliver competitive accuracy but emit 65k dimensional descriptors and demand over 50MB of memory. iSQRT-COV halves that footprint to 32k dimension yet still occupies 25MB.  CBP is far more compact (8k dimensions, $<$1MB) but loses more than ten percentage points of R@1 on every split. \textit{CPS ($k=2$) combines the best of both worlds}: at 8k dimension it matches CBP’s compactness while restoring full-covariance accuracy, achieving 95.0\% R@1 on Oxford and leading every baseline on the In-house regions.  These results establish CPS ($k=2$) as the most size-efficient and, to the best of our knowledge, highest-performing SoA method. %
In Fig. \ref{fig: top_fig}, we compare the \coolname{} performance on \wildplaces{} dataset, where in addition to SoA methods, we include a comparison with a self-attention-based global feature aggregation approach \citep{lee2019set} since it implicitly models feature correlations. The results demonstrate the superiority of CPS (across $k=2,4,8$) over the self-attention mechanism.

\paragraph{Study 5: Backbone-agnostic integration of \coolname{} ($k=2$):}
To verify that the benefits of \coolname{} ($k=2$) generalize beyond a single encoder, we replace the GeM pooling layer in MinkLoc3Dv2 \citep{komorowski2022improving} with \coolname{} ($k=2$), leaving all other architectural components, training schedules, and the AP loss \citep{brown2020smooth} unchanged. This swap isolates the impact of the pooling module on a high-capacity backbone. 

On Oxford and the three In-house regions, Table~\ref{tab:oxford_baseline_refined_results} shows that \coolname{} ($k=2$) + MinkLoc3Dv2 achieves the highest R@1 and R@1\% in 14 of 16 metric–dataset pairs, outperforming the GeM variant by up to 1.3\,pp and surpassing all prior methods, including TransLoc3D and SVT-Net.  Given that these scans are downsampled to 4096 points, the gains confirm  \coolname{} ($k=2$)’s ability to extract richer descriptors under sparse input conditions. On the MulRan benchmark (Table~\ref{tab: mulran-results}),  \coolname{} ($k=2$) raises R@1 on Sejong from 73.1\% (MinkLoc3Dv2 + GeM) to 91.9\%, and on the unseen DCC city from 62.2\% to 68.6\%, with R@5 reaching 97.4\% and 91.6\% respectively.  These substantial improvements in an unseen environment highlight  \coolname{} ($k=2$)’s strong generalization capacity in complex urban scenarios. Finally, on WildPlaces (Table~\ref{tab: wild-places results}),  \coolname{} ($k=2$) attains 80.22\% R@1 on Venman and 75.04\% on Karawatha, surpassing the previous best \loggnet{} even while using the standard AP loss and training twice as fast on a single GPU.

These results demonstrate that \coolname{} is \emph{backbone-agnostic}: by swapping a single pooling layer, we elevate MinkLoc3Dv2 to state-of-the-art performance across four diverse datasets, urban, campus, mixed-traffic, and natural trails, without additional tuning. Due to space limitations, we give qualitative results in the Supplementary Materials that include feature visualizations.

\begin{table*}[t]
\setlength{\tabcolsep}{1pt}
\centering
\caption{Evaluation results of Place Recognition methods on Oxford and In-house datasets using the Baseline and Refined protocols from \cite{uy2018pointnetvlad}. R@1: Recall@1; R@1\%: Recall@1\%.}
\label{tab:oxford_baseline_refined_results}
\resizebox{\linewidth}{!}{%
\begin{tabular}{%
  l  %
  l  %
  *{8}{c}  %
  *{8}{c}  %
}
\toprule
\multirow{4}{*}{\textbf{LPR Method}} 
  & \multirow{4}{*}{\textbf{Agg.}}
  & \multicolumn{8}{c}{\textbf{Baseline protocol}}
  & \multicolumn{8}{c}{\textbf{Refined protocol}} \\
\cmidrule(lr){3-10} \cmidrule(lr){11-18}
  & 
  & \multicolumn{2}{c}{\textbf{Oxford}}
  & \multicolumn{2}{c}{\textbf{U.S.}}
  & \multicolumn{2}{c}{\textbf{R.A.}}
  & \multicolumn{2}{c}{\textbf{B.D.}}
  & \multicolumn{2}{c}{\textbf{Oxford}}
  & \multicolumn{2}{c}{\textbf{U.S.}}
  & \multicolumn{2}{c}{\textbf{R.A.}}
  & \multicolumn{2}{c}{\textbf{B.D.}} \\
\cmidrule(lr){3-4} \cmidrule(lr){5-6} \cmidrule(lr){7-8} \cmidrule(lr){9-10}
\cmidrule(lr){11-12} \cmidrule(lr){13-14} \cmidrule(lr){15-16} \cmidrule(lr){17-18}
  & 
  & \textbf{R@1} & \textbf{R@1\%}
  & \textbf{R@1} & \textbf{R@1\%}
  & \textbf{R@1} & \textbf{R@1\%}
  & \textbf{R@1} & \textbf{R@1\%}
  & \textbf{R@1} & \textbf{R@1\%}
  & \textbf{R@1} & \textbf{R@1\%}
  & \textbf{R@1} & \textbf{R@1\%}
  & \textbf{R@1} & \textbf{R@1\%} \\
\midrule
PointNetVLAD \cite{uy2018pointnetvlad} & \netvlad
  & 62.8 & 83.0 & 63.2 & 72.6 & 56.1 & 60.3 & 57.2 & 65.3
  & 63.3 & 80.1 & 86.1 & 94.5 & 82.7 & 93.1 & 80.1 & 86.5 \\

PCAN \cite{zhang2019pcan}           & \netvlad
  & 70.3 & 86.4 & 73.7 & 89.1 & 58.1 & 69.1 & 66.8 & 75.2
  & 70.7 & 86.4 & 83.7 & 94.1 & 82.5 & 92.5 & 80.3 & 87.0 \\

LPD-Net \cite{liu2019lpd}           & \netvlad
  & 86.3 & 94.3 & 87.0 & 96.0 & 79.1 & 85.7 & 82.5 & 89.1
  & 86.6 & 94.9 & 94.4 & 98.9 & 90.8 & 96.4 & 90.8 & 94.4 \\

SOE-Net \cite{xia2021soe}           & \netvlad
  & 89.4 & 96.5 & 93.7 & 96.7 & 90.2 & 92.4 & --   & --
  & 89.3 & 96.4 & 91.8 & 97.7 & 90.2 & 95.9 & 89.0 & 92.6 \\

MinkLoc3D \cite{komorowski2021minkloc3d} & GeM
  & 93.0 & 97.4 & 86.7 & 97.5 & 90.4 & 91.5 & 81.5 & 90.1
  & 94.8 & 98.5 & 97.2 & 99.7 & 96.7 & 99.3 & 94.0 & 96.7 \\

PPT-Net \cite{zhou2021ndt}          & GeM
  & 93.5 & 97.8 & 90.1 & 97.5 & 89.7 & 95.1 & 88.4 & 90.7
  & --   & 98.4 & --   & 99.7 & --   & 99.5 & --   & 95.3 \\

SVT-Net \cite{fan2022svt}           & GeM
  & 95.0 & 98.5 & 94.9 & 97.4 & 90.4 & 94.9 & 89.3 & 93.3
  & 94.7 & 98.4 & 97.0 & 99.9 & 95.2 & 99.5 & 94.4 & 97.2 \\

TransLoc3D \cite{xu2021transloc3d}   & \netvlad
  & 95.0 & 98.3 & \textbf{95.4} & 97.6 & \textbf{91.0} & 94.7 & 88.4 & 94.7
  & 95.0 & 98.5 & 97.5 & 99.8 & 94.4 & 99.7 & 94.8 & 97.4 \\

MinkLoc3Dv2 \cite{komorowski2022improving} & GeM
  & 95.8 & 98.7 & 90.0 & 96.5 & 84.9 & 91.7 & 84.9 & 89.7
  & 96.8 & 98.9 & 98.9 & 100  & 98.4 & 99.7 & 97.2 & 98.7 \\
\midrule

\rowcolor{green!15}
\textbf{MinkLoc3Dv2}                     & \coolname{}
  & \textbf{97.4} & \textbf{99.2} & 93.5 & \textbf{98.0} & 89.5 & \textbf{94.9} & \textbf{89.8} & \textbf{93.9}
  & \textbf{98.1} & \textbf{99.3} & \textbf{98.9} & \textbf{99.9} & \textbf{99.1} & \textbf{99.7} & \textbf{98.8} & \textbf{99.6}\\
\bottomrule
\end{tabular}}%
\end{table*}

\begin{table}[t]
  \centering
  \begin{minipage}[t]{0.48\linewidth}
    \setlength{\tabcolsep}{4pt}
    \renewcommand{\arraystretch}{0.925}
    \centering
    \caption{Evaluation Results with 5m thresholds on Sejong and DCC environments of MulRan dataset using the training and testing protocols proposed in [43]. The results of existing methods are quoted from \citep{komorowski2021egonn}. R@1/5: Recall@1/5.}
    \label{tab: mulran-results}
    \resizebox{\linewidth}{!}{%
      \begin{tabular}{l l c c c c}
        \toprule
        \multirow{2}{*}{\textbf{Method}}  
          & \multirow{2}{*}{\textbf{Agg.}}  
          & \multicolumn{2}{c}{\textbf{Sejong}} 
          & \multicolumn{2}{c}{\textbf{DCC}} \\ 
        \cmidrule(lr){3-4}\cmidrule(lr){5-6}
         &  & \textbf{R@1} & \textbf{R@5}   
               & \textbf{R@1} & \textbf{R@5} \\ 
        \midrule
        Locus    \cite{vidanapathirana2021locus} 
          & SoP  & 67.0 & 75.8 & 46.3 & 55.6 \\
        PPT-Net  \cite{zhou2021ndt} 
          & GeM  & 60.2 & 76.2 & 47.9 & 60.2 \\
        MinkLoc3Dv2 \cite{komorowski2022improving} 
          & GeM  & 73.1 & 86.4 & 62.2 & 73.2 \\
        LCD-Net  \cite{cattaneo2022lcdnet} 
          & \netvlad 
                & 63.1 & 82.0 & 57.7 & 71.6 \\
        \midrule
        \rowcolor{green!15}
        \textbf{MinkLoc3Dv2} 
          & \coolname{} 
                & \textbf{91.9} & \textbf{97.4} 
                & \textbf{68.6} & \textbf{91.6} \\ 
        \bottomrule
      \end{tabular}%
    }
  \end{minipage}\hfill
  \begin{minipage}[t]{0.49\linewidth}
    \setlength{\tabcolsep}{4.2pt}
    \renewcommand{\arraystretch}{0.95}
    \centering
    \caption{Evaluation results on the WildPlaces dataset using the training and testing protocols proposed in the dataset \citep{knights2023wild}. The results of existing methods are quoted from \citep{knights2023wild}. R@1: Recall@1; MRR: Mean Reciprocal Rank.}
    \label{tab: wild-places results}
    \resizebox{\linewidth}{!}{%
      \begin{tabular}{l l c c c c}
        \toprule
        \multirow{2}{*}{\textbf{Method}}  
          & \multirow{2}{*}{\textbf{Agg.}}  
          & \multicolumn{2}{c}{\textbf{Venman}} 
          & \multicolumn{2}{c}{\textbf{Karawatha}} \\ 
        \cmidrule(lr){3-4}\cmidrule(lr){5-6}
        &  & \textbf{R@1}  & \textbf{MRR}   
              & \textbf{R@1}  & \textbf{MRR}\\ 
        \midrule
        ScanContext \cite{kim2018scan} 
          & --     & 33.98 & 64.67 & 38.44 & 67.90 \\ 
        TransLoc3D \cite{xu2021transloc3d} 
          & \netvlad{} 
                   & 50.24 & 66.16 & 46.08 & 50.24 \\ 
        MinkLoc3Dv2 \cite{komorowski2022improving} 
          & GeM    & 75.77 & 84.87 & 67.82 & 79.21 \\ 
        LoGG3D-Net \cite{vidanapathirana2022logg3d} 
          & SoP    & 79.84 & 87.33 & 74.67 & 83.68 \\ 
        \midrule
        \rowcolor{green!15}
        \textbf{MinkLoc3Dv2}  
          & \coolname{} 
                   & \textbf{80.22} & \textbf{87.36} 
                   & \textbf{75.04} & \textbf{83.73} \\
        \bottomrule
      \end{tabular}%
    }
  \end{minipage}
\end{table}

%% file: chapters/conclusion.tex
\section{Conclusion}
Our five-part study shows that second-order aggregation is the most bandwidth-efficient way to build global LiDAR descriptors: full covariance already outperforms strong first-order baselines, yet remains remarkably robust when the feature budget is cut by an order of magnitude. Channel-Partitioned SoA (\coolname{}) preserves that robustness while reducing the global descriptor dimension by 4–16 times; a single setting \coolname{} $(k=2)$ emerges as the sweet-spot, matching full covariance accuracy with only 8k dimensions and 13 MB of scratch memory. \coolname{} $(k=2)$ generalizes seamlessly from lightweight MinkLoc3D to the high-capacity MinkLoc3Dv2 backbone, lifting it to new state-of-the-art results on four benchmarks spanning urban streets, campus routes, complex city loops, and dense forest trails. Because adopting \coolname{} requires adjusting only the pooling layer, these gains come with negligible architecture variations overhead. %
In the future, we plan to integrate \coolname{} with transformer backbones to achieve even higher performance and test on other LPR datasets. We also plan to test it on resource constrained systems.

%% file: chapters/supplementary.tex
\clearpage

\begin{nolinenumbers}
\begin{center}
    \LARGE \textbf{Supplementary Material} \\[0.5em]
    A Deeper Look into Second-Order Feature Aggregation for LiDAR Place Recognition
\end{center}
\end{nolinenumbers}

\setcounter{linenumber}{1}
\setcounter{page}{1}
\setcounter{section}{0}

\section{Evaluation of Kernel Matrix based Second-order Feature Aggregation}

In this section, we provide the evaluation of kernel matrix based second-order feature aggregation \citep{zhang2022kernelized} mentioned in line 113 of the main text. 

\begin{wrapfigure}[14]{r}{0.40\linewidth}
\vspace{-12pt}
\centering
\includegraphics[width=\linewidth]{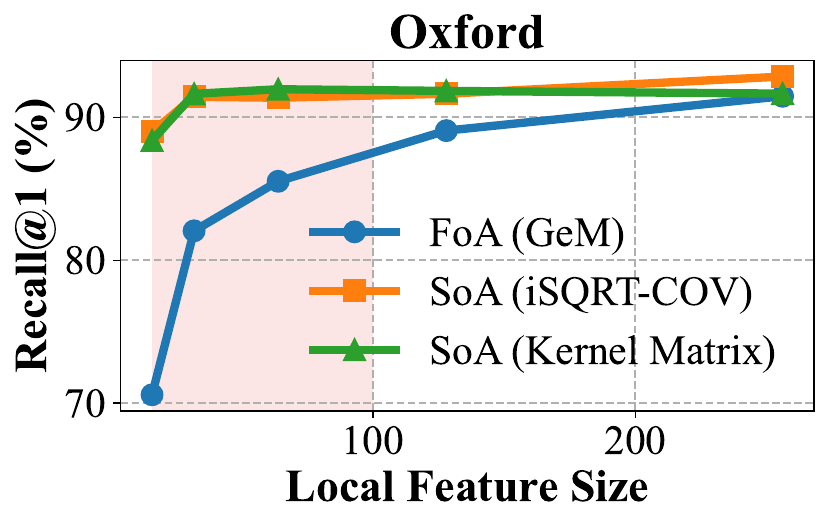}
\caption{Performance of GeM, iSQRT-COV, Kernel Matrix when local feature channels $d$ are reduced with $1\times 1$ convolution to following dimensions: \ie{} $d=\{16, 32, 64, 128, 256\}$.}
\label{fig: kernel matrix comparison}
\vspace{-10pt}
\end{wrapfigure}

For consistency of understanding, we compare kernel matrix aggregation with GeM and iSQRT-COV performance shown Fig. \ref{fig: performance-vs-descriptor-count} under the same experimental setting, \ie{} inserting a learnable $1\times1$ convolution after MinkLoc3D’s final ResNet block to project the 256-d feature map down to $d=\{16,32,64,128\}$. Fig. \ref{fig: kernel matrix comparison} shows the results. We use a RBF kernel following \citet{zhang2022kernelized} with default width of 1 without tuning it further. We used same matrix normalization method as iSQRT-COV since kernel matrix is also symmetric positive definite (SPD).

The results show minor improvement in retrieval performance in all cases except $d=16$ and $d=256$ over the performance of iSQRT-COV. Tuning the RBF kernel width with grid search or automatic learning as done in \citep{engin2018deepkspd} may further improve the performance, however, we leave this exploration for the future works.

\section{Covariance Matrix is Permutation Invariant}

In this section, we give the detail of permutation invariant property of covariance matrices mentioned in line 140 of the main text.

\begin{lemma}[Symmetry of covariance pooling]\label{lem:covariance_symmetry}
Covariance pooling is permutation invariant: it returns the same mean and covariance no matter how the input descriptors are ordered.
\end{lemma}

\begin{proof}
Let \(P=\{p'_1,\dots,p'_N\}\subset\mathbb{R}^d\) and define the empirical statistics
\[
\mu=\frac1N\sum_{t=1}^{N}p'_t,
\qquad
\Sigma=\frac1N\sum_{t=1}^{N}(p'_t-\mu)(p'_t-\mu)^{\top}.
\]

For any permutation \(\pi\) of \(\{1,\dots,N\}\) set
\(\widetilde P=\{p'_{\pi(1)},\dots,p'_{\pi(N)}\}\).
Commutativity of addition implies
\[
\widetilde\mu=\frac1N\sum_{t=1}^{N}p'_{\pi(t)}=\mu,
\]
so the centred multisets \(\{p'_t-\mu\}\)  
\(\{p'_{\pi(t)}-\widetilde\mu\}\) coincide.  Therefore
\[
\widetilde\Sigma
=\frac1N\sum_{t=1}^{N}
      (p'_{\pi(t)}-\widetilde\mu)(p'_{\pi(t)}-\widetilde\mu)^{\top}
=\frac1N\sum_{t=1}^{N}
      (p'_t-\mu)(p'_t-\mu)^{\top}
=\Sigma.
\]
Both \(\mu\) and \(\Sigma\) are thus unaffected by \(\pi\), proving the claim.\qedhere
\end{proof}

\section{Datasets}
In this section, we give the details of datasets mentioned in line 202 of the main text.

We conduct extensive experiments with four large-scale, public datasets for LiDAR place recognition (Oxford Robotcar, In-house, \mulran{}, \wildplaces{}) to demonstrate the effectiveness of \coolname{}. The details are given below.

\textbf{Oxford Robotcar} \cite{maddern20171} is the most widely used dataset for LiDAR place recognition. It has LiDAR scans captured by traveling a route of 44 times ($\approx$10 km) across Oxford, UK, over a year. The dataset is evaluated by taking the point clouds of one trip as queries and iteratively matched against the point clouds of other trips. In this work, we follow the training and testing splits introduced by Uy et al. \cite{uy2018pointnetvlad}. A total of 24.7k point clouds were used for training and testing.

\textbf{In-house}  \cite{uy2018pointnetvlad} is also a popular dataset in the literature. It has LiDAR scans captured by traveling a route of 5 times ($\approx$10 km) across three regions of Singapore -- a Business District (B.D.), a Residential Area (R.A.), and a University Sector (U.S.). Similar to the Oxford dataset, we follow the testing split introduced by Uy et al. \cite{uy2018pointnetvlad} and use the point clouds from a single trip as queries and the remaining point clouds from other trips as databases iteratively. However, unlike the Oxford dataset, the in-house dataset is only used for testing purposes to demonstrate generalisability. A total of $\approx$4.5k point clouds were used for testing.

\textbf{\mulran}~\cite{kim2020mulran} has LiDAR scans captured by traveling through various urban environments. Among those, we use the traversals of Sejong City (Sejong) and Daejeon Convention Center (DCC) (each 3 runs $\approx$15 km). Following the work of \cite{komorowski2021egonn,vidanapathirana2023spectral}, we produce the training and testing splits of both environments and train using only Sejong sequences 1 and 2. The evaluation is done on test sets of both environments.

\textbf{\wildplaces}~\cite{knights2023wild} is a very recent dataset. Unlike the above datasets, it is captured in natural, forest environments. It has LiDAR scans of 33km captured from hiking trails in Brisbane, Australia over 14 months. Following its original paper, we evaluate Venman and Karawatha environments using the inter-sequence training and testing protocol.

\section{Visualization of Features}
In this section, we give visualizations of the features mentioned in line 316 of the main text.

\begin{figure}
    \centering
    \includegraphics[width=1\linewidth]{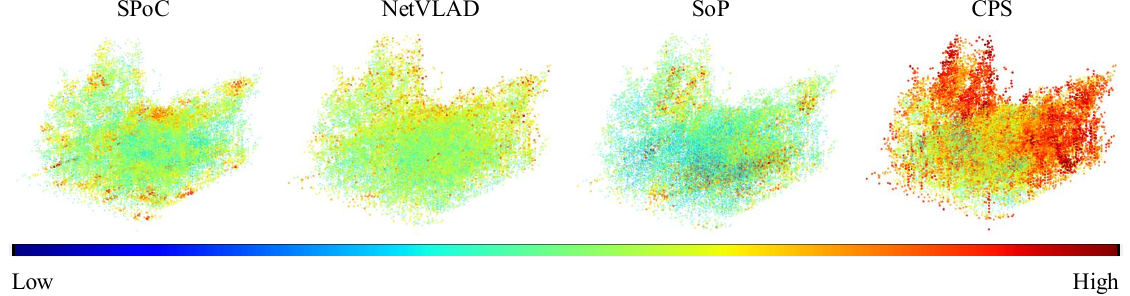}
    \caption{Visualization of features learned by Minkloc3D \cite{komorowski2021minkloc3d} using SPoC \citep{radenovic2018fine}, NetVLAD \citep{arandjelovic2016netvlad}, SoP \cite{vidanapathirana2022logg3d} and our CPS feature aggregators on a LiDAR scan from the \wildplaces{} (Karawatha environment) dataset. The overlaid heatmap is computed as the average activation across all channels at the last backbone layer, which produces coarse “blob” shapes due to reduced spatial resolution at that stage. Colors range from blue (low activation) to red (high activation) as shown in the legend.}
    \label{fig: feature-visulaizations}
\end{figure}

Fig. \ref{fig: feature-visulaizations} presents qualitative feature maps for SPoC \citep{radenovic2018fine}, NetVLAD \citep{arandjelovic2016netvlad}, SoP \cite{vidanapathirana2022logg3d} and our CPS aggregation on an example LiDAR scan from the Karawatha environment of the \wildplaces{} dataset. Heatmaps encode the average activation across all channels in the final backbone layer, yielding a coarser point resolution as noted in the caption. In this example, CPS produces generally higher activation values in semantically meaningful regions than the other three methods. The quantitative comparison of these methods is reported in Fig. \ref{fig: top_fig}.